\documentclass{article}
\usepackage[preprint]{log_2024}			% for preprint version

\usepackage{booktabs}						% professional-quality tables
\usepackage{multirow}						% tabular cells spanning multiple rows
\usepackage{amsfonts}						% blackboard math symbols
\usepackage{graphicx}						% figures
\usepackage{duckuments}						% sample images
\usepackage{amsthm}
\usepackage[numbers,compress,sort]{natbib}	% for numerical citations

\newcommand{\lrVert}[1]{\left\|#1\right\|}
\newcommand{\first}[1]{\textbf{\textcolor{red}{#1}}}
\newcommand{\second}[1]{\textbf{\textcolor{violet}{#1}}}
\newcommand{\third}[1]{\textbf{\textcolor{orange}{#1}}}

\newtheorem{theorem}{Theorem}
\newtheorem{lemma}{Lemma}
\newtheorem{definition}{Definition}

\title[Range-aware Positional Encoding via High-order Pretraining: Theory and Practice]{Range-aware Positional Encoding via\\ High-order Pretraining: Theory and Practice}

\author[Nguyen et al.]{%
Viet Anh Nguyen\\
FPT Software AI Center\\
\email{anhnv117@fpt.com}\And
Nhat Khang Ngo\\
Mila - Quebec AI Institute, McGill University\\
\email{khang.ngo@mila.quebec}\And
Truong Son Hy\thanks{Corresponding Author}\\
University of Alabama at Birmingham\\
\email{thy@uab.edu}
}

\begin{document}

\maketitle

\begin{abstract}
Unsupervised pre-training on vast amounts of graph data is critical in real-world applications wherein labeled data is limited, such as molecule properties prediction or materials science. Existing approaches pre-train models for specific graph domains, neglecting the inherent connections within networks. This limits their ability to transfer knowledge to various supervised tasks. In this work, we propose a novel pre-training strategy on graphs that focuses on modeling their multi-resolution structural information, allowing us to capture global information of the whole graph while preserving local structures around its nodes. We extend the work of \textbf{Wave}let \textbf{P}ositional \textbf{E}ncoding (WavePE) from \citet{10.1063/5.0152833} by pretraining a \textbf{H}igh-\textbf{O}rder \textbf{P}ermutation-\textbf{E}quivariant Autoencoder (HOPE-WavePE) to reconstruct node connectivities from their multi-resolution wavelet signals. Unlike existing positional encodings, our method is designed to become sensitivity to the input graph size in downstream tasks, which efficiently capture global structure on graphs. Since our approach relies solely on the graph structure, it is also domain-agnostic and adaptable to datasets from various domains, therefore paving the wave for developing general graph structure encoders and graph foundation models. We theoretically demonstrate that there exists a parametrization of such architecture that it can predict the output adjacency up to arbitrarily low error. We also evaluate HOPE-WavePE on graph-level prediction tasks of different areas and show its superiority compared to other methods. 
% We will release our source code upon the acceptance.
% Our source code is publicly available at \url{https://github.com/HySonLab/WaveletPE}.

% introduce a novel positional encoding scheme that is based on graph wavelet theory. Our wavelet positional encoding (WavePE) can be computed more efficiently than the two most popular graph positional encoding schemes that use Laplacian (LapPE) and random walk (RWPE). We theoretically prove that WavePE can replace LapPE and RWPE. Based on this analysis, we further present our attempt to pre-train a high-order equivariant auto-encoder that learns to reconstruct the graph wavelet signals at multiple scales. This paves the way for generalizing graph structural encoders that can be effectively adapted to different domain-specific datasets, thereby enabling the development of large graph foundation models.

% In this paper, we introduce a novel wavelet positional encoding (PE) on graphs and show that it can replace the two most popular graph positional encodings such as Laplacian PE (LapPE) and Random walk PE (RWPE). 
\end{abstract}

\section{Introduction} \label{sec:intro}

% **** General introduction to graph representation learning ****
One of the fastest-growing areas in machine learning is graph representation learning, with impactful applications in biomedicine, molecular chemistry, and network science. Most graph neural networks (GNNs) rely on the message-passing framework that processes graph-structured data by exchanging the vectorized information between nodes on graphs along their edges. Albeit achieving remarkable results in a wide range of tasks on graph data, message-passing neural networks (MPNNs) possess several fundamental limits, including expressiveness \cite{GraphRepresentationLearning:4761, thiede2020general, HyEtAl2018, hy2019covariant, gin}, over-smoothing \cite{chen-oversmoothing}, and over-squashing \cite{topping2022understanding}. In recent years, transformer-based architectures  \cite{graph_transformer_networks, san, generalization_trans, graphormer, 10.1063/5.0152833} have emerged as powerful alternatives to address the mentioned issues of MPNN. The self-attention mechanism in conventional transformers computes the pairwise interactions between the input tokens, enabling the modeling of long-range interactions between distant tokens and overcoming information bottlenecks on graphs \cite{cai2023connection, trang2024scalable, Hy_2023}. While applying transformers to graphs offers advantages, it often necessitates a trade-off between computational resources and performance, particularly when scaling models to massive datasets with thousands to millions of nodes, like citation networks. Augmenting virtual nodes (VN) to the original graph's nodes , which allows global connections among these nodes, has emerged as a promising strategy to balance these two objectives \cite{cai2023connection, Hy_2023, pmlr-v238-luo24a, 10.5555/3305381.3305512}.
\par \vspace{10pt}
% **** positional encoding ***** ####
However, graph transformers (GTs) and VN-augmented MPNNs disregard the underlying structure of graph data by altering inherent connections among the nodes (i.e., shortening all paths to two). This disregard may explain their limitations in several graph-level prediction tasks. To address this, positional and structural encodings (PSE) are commonly used to enhace structural information in modern GNNs. However, existing PSE encoding methods often have limitations in terms of task specificity. For instance, random walk encodings \citep{rw, RWPE} excel with small molecules, while Laplacian encoding \cite{generalization_trans} captures long-range interactions effectively. Meanwhile, \citet{10.1063/5.0152833} propose Wavelet positional encoding (WavePE), a PSE encoding technique that models the node interactions at multiple scales. By levraging multi-resolution analysis based on Wavelet Transform on graphs \cite{wavelet, hammond2011wavelets, wavelet2018, pmlr-v196-hy22a, pmlr-v228-nguyen24a}, WavePE enables neural networks to capture a comprehensive range of structural information, from local to global properties.

% **** self supervised for graph learning ******
\vspace{10pt}
Real-world applications of learning methods on graphs face two challenges \cite{transfer_learning, pmlr-v97-hendrycks19a}. First, in fields like biology and chemistry, task-specific labeled data is scarce, and acquiring them requires significant time and resources \cite{Zitnik2018}. Second, real-world graph datasets often contain out-of-distribution samples; for example, newly synthesized molecules may have different structures from those in the training set, leading to inaccurate prediction of supervised learning models. Meanwhile, transfer learning is a promising approach to overcoming these challenges. Our key insight is that graph data possesses two distinct feature types: (1) domain-specific (e.g., atom types in molecules or user names citation networks) and (2) topological features. While the former depends on specific domains, the latter is a general form of graph data. This observation gives rise to our idea of \textit{structural pretraining}, wherein we pretrain a model on topological features of graph data. Specifically, given the Wavelet signals of a graph, we train an autoencoder to reconstruct its adjacency matrix. The graph wavelet signals are defined based on functions of eigenvalues of the graph Laplacian at \textit{multiple resolutions}, representing general forms of connectivities among nodes on the graphs \cite{pmlr-v196-hy22a, pmlr-v228-nguyen24a, 10.1063/5.0152833}. As a result, a model trained on these features should be able to capture the underlying patterns of different graph structures and generalize well to downstream tasks. After pretraining, we can use this pretrained model as a general structural encoder that extracts node structural features, which are passed to graph neural networks along with domain-specific features for predicting task-specific properties. 
\vspace{-5pt}
\subsection{Contribution} 
In this work, we propose a new pretraining approach on graphs that leverages the reconstruction of graph structures from the Wavelet signals to generalize structural information on graph data, thus enabling transfer learning to various downstream tasks across various domains of different ranges. Our contributions are three-fold as follows:
\begin{itemize}
    \item[\textbullet] We propose a high-order structural pretrained models for graph-structured data and a loss-masking technique that leverages high-order interactions of nodes on graphs while being aware of the graph size and diameter, therefore capturing better positional and structural information. 
    \item[\textbullet] We theoretically prove that pretraining by reconstruction with multi-resolution Wavelet signals can make autoencoder learn node state after an arbitrarily walk of length $d$, which can contain both local and global information of graph structures. We also analyzed the width required for the latent space of the autoencoder to ensure such performance and also propose a low-rank alternative.
    \item[\textbullet] We empirically show that such pretraining scheme can enhance the performance of supervised models when fine-tuned on downstream datasets of different domains, indicating the generalizabiilty and effectiveness of pretrained structural encoding compared with other domain-specific pretraining methods. 
\end{itemize}

\section{Related work} \label{sec:related}
\paragraph{Graph Positional and Structural Encodings} Node positional encodings are augmented into node features to preserve the graph structure information and increase the expressiveness of MPNNs \cite{signnetbasisnet, rw, wang2022equivariant, JMLR:v24:22-0567} and Graph Transformers \cite{graph_generalize, RWPE, menegaux2023selfattention, kim_pure, gps, graph_transformer_networks}. Many recent works directly use the graph Laplacian eigenvectors as node positional encodings and combine them with domain-specific node features \cite{graph_generalize, kim_pure, gps, graph_transformer_networks}. Other works use several approaches to encode graph's structural information, including random walks \cite{mialon2021graphit, RWPE, rw, li2020distance}, diffusion kernels on graphs \cite{mialon2021graphit, san, 10.1063/5.0152833}, shortest paths \cite{li2020distance, ying-2021}, and unsupervised node embedding methods \cite{wang2022equivariant, liu2023graph}. Our work builds on similar approaches to \citet{wang2022equivariant} and \citet{liu2023graph}, where we pre-train an encoder to capture node positional information in an unsupervised setting. However, these pretraining methods are either struggle to capture long-range information, lack equivariant constraints that is intrinsic on graphs, or being domain-specific on certain datas and fail in terms of transferrability. We overcome both of these shortcomings by designing a permutation equvariant autoencoder that only focus on the intrinsic structure of the graph, neglecting domain-specific features while being able to learn long hops neighbor in a graph. These learned positional features can be adapted to various downstream tasks and generalize well to domain-specific datasets. Furthermore, we theoretically demonstrate that our method is general and can express several common positional encoding schemes.

% \paragraph{Graph Sparsification} 
\paragraph{Pretraining on Graphs} In the age of modern deep learning, pre-training models on massive unlabeled datasets and then fine-tuning them on smaller, labeled datasets has proven remarkably successful in areas like natural language processing \cite{brown2020language, team2023gemini} and computer vision \cite{he2020momentum, radford2021learning}. For graph-structured data, self-supervised pre-training is a growing area of research, with contrastive learning being a popular approach  \cite{pmlr-v139-you21a, pmlr-v139-xu21g, zhu2021empirical, You2020GraphCL, jiao2020sub}. While these methods can enhance performance on various downstream tasks, their reliance on domain-specific features to create different views of the input graphs limits their transferability to other domains. For instance, models pretrained on small molecules struggle to adapt to other graph types like citation networks, proteins, or code repositories. Our work aims to propose a more generalizable graph pre-training method that relies only on graph \textit{intrinsic features}, i.e. adjacency matrices, which can be trained on unfeatured data and easily transferred for predicting properties of
small graph datasets in arbitrary domains. To solidify such transferrability, we elevate a pretraining method that learn a spectrum of range on graphs. This way, structures on different domains belonging to the spectrum can still be captured without being reliant on domain features.
\vspace{-15pt}
\section{Notations and Preliminaries}
\subsection{Notations}
In this work, we define a graph $G = (V, E, \bold{A})$, where $V$ and $E$ denote node and edge sets, respectively, and $\bold{A} \in \mathbb{R} ^ {n\times n}$ is the adjacency matrix. The normalized graph Laplacian is computed as: $\widetilde{L} = \bold{I}_n - \bold{D}^{-1/2} \bold{A} \bold{D}^{-1/2}$ where $\bold{D}$ is the diagonal degree matrix, and $\bold{I}_n$ is the identity matrix.

\subsection{Permutation equivariant function}
In this section, we formally define permutation symmetry, i.e. symmetry to the action of the symmetric group, $\mathbb{S}_n$, and construct permutation-equivariant neural networks to encode graph wavelets. An element $\sigma \in \mathbb{S}_n$ is a permutation of order $n$, or a bijective map from $\{1,2,\ldots, n\}$ to $\{1,2,\ldots, n\}$. We present each element $\sigma$ in $\mathbb{S}_n$ as a permutation matrix $\bold{P}_{\sigma}\in\mathbb{R}^{n\times n}$. For example, the action of $\mathbb{S}_n$ on an adjacency matrix $\bold{A} \in \mathbb{R}^{n \times n}$ and on a latent matrix (i.e. node embedding matrix) $\bold{Z} \in \mathbb{R}^{n \times d_z}$ are:
\[
\sigma(\bold{A}) = \bold{P}_{\sigma}\bold{A}\bold{P}_{\sigma}^{\top}, \ \ \ \ 
\sigma(\bold{Z}) = \bold{P}_{\sigma}\bold{Z},
\]
for $\sigma \in \mathbb{S}_n$. Here, the adjacency matrix $\bold{A}$ is a second-order tensor with a single feature channel, while the latent matrix $\bold{Z}$ is a first-order tensor with $d_z$ feature channels. LIkewise, the permutation on a $k$-th order tensor $\bold{X}\in\mathbb{R}^{n^k\times d}$ operate on the first $k$ dimensions of $\bold{X}$.

\noindent
Furthermore, we define permutation equivariance on non-homogeneous order functions.

\begin{definition} \label{def:Sn-equivariant}
A $\mathbb{S}_n$-equivariant (or permutation equivariant) function is a function $f\colon \mathbb{R}^{n^k \times d} \to \mathbb{R}^{n^{k'} \times {d'}}$ that satisfies $f(\sigma(\bold{X})) = \sigma(f(\bold{X}))$ for all $\sigma \in \mathbb{S}_n$ and $\bold{X} \in \mathbb{R}^{n^k \times d}$. 
\end{definition} 

% Although Def. \ref{def:Sn-equivariant} is generalized for all order features, in the scope of this work we will only include up to the second-order feature, i.e. $k,k'\leq 2$. This definition is also applicable to neural networks, whose all components satisfy such properties
\section{Methodology} \label{sec:wavelet_pe}

\begin{figure}
     \centering
     % \begin{subfigure}
         \centering
         \includegraphics[width=0.25\textwidth]{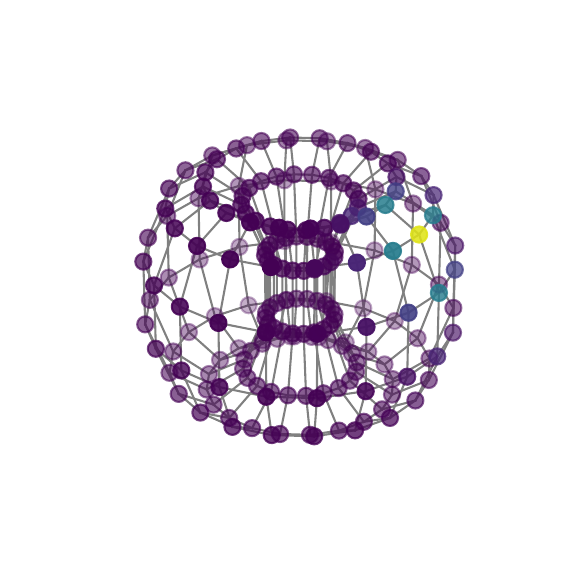}
         % \vspace{-35pt}
         % \caption*{$s=4$}
         %\label{fig:y equals x}
     % \end{subfigure}
     % \begin{subfigure}
         % \centering
         \includegraphics[width=0.25\textwidth]{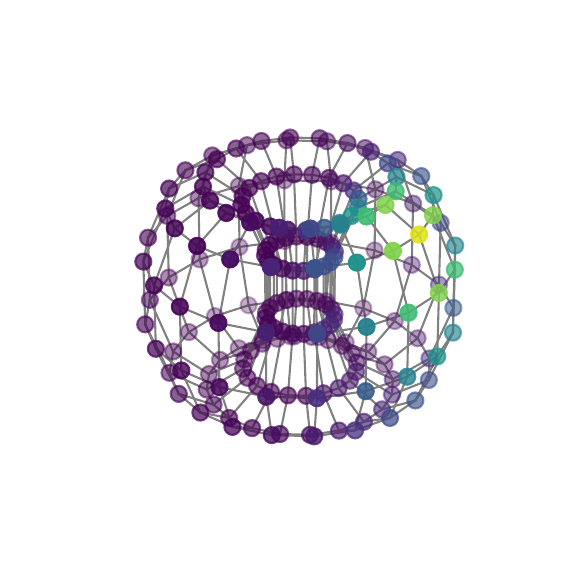}
         % \vspace{-35pt}
         % \caption*{$s=15$}
         %\label{fig:three sin x}
     % \end{subfigure}
     % \begin{subfigure}
         % \centering
         \includegraphics[width=0.25\textwidth]{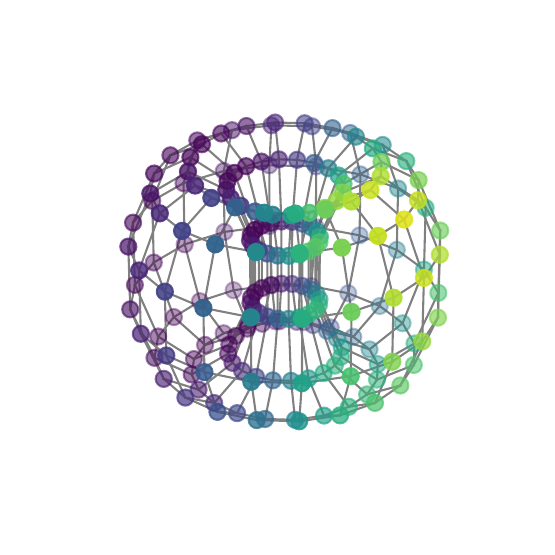}\\
         % \vspace{-35pt}
         % \caption*{ $s=50$}
         %\label{fig:five over x}
     % \end{subfigure}
     % \hfill 
     % \begin{subfigure}
     
     \includegraphics[width=0.35\textwidth]{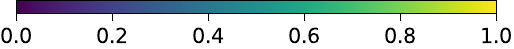}
     
     % \end{subfigure}
        \caption{Visualization of graph Wavelet on the geometric graph of a torus. The low scaling factor $s$ results in a highly localized structure around the center node (yellow), while higher factors can lead to smoother signals that can spread out to a larger part of the graph with scaling factor 4, 15 and 50 respectively.}
        \label{fig:wavelet}
\end{figure}

\subsection{Spectral Graph Wavelet Tensors}
As the normalized graph Laplacian $\widetilde{L} \in \mathbb{R}^{n \times n}$ is symmetric when $G$ is undirected, we can decompose it into a complete set of orthonormal eigenvectors $U = (u_1, u_2, \cdots, u_n)$ wherein $u_i$ is associated with a real and non-negative eignevalue $\lambda_i$ as:

\begin{equation}
    \widetilde{L} = U \Lambda U^T,
    \label{eq:eigen_decomposition}
\end{equation}
where $\Lambda = \text{diag}(\lambda_1, \cdots, \lambda_n)$ is the diagonal matrix of eigenvalues. The graph Wavelet transforms construct a set of spectral graph Wavelet as bases to project the graph signal from the vertex domain to the spectral domain as: 
\begin{equation}
    \psi_s = U \Lambda_s U^T, 
    \label{eq:wavelet}
\end{equation}
here, $\Lambda_s = \text{diag}(g(s\lambda_1), \cdots, g(s\lambda_n))$ is the scaling matrix of the eigenvalues. The scaling function $g$ takes the eigenvalue $\lambda_i$ and an additional scaling factor $s$ as inputs, indicating how a signal diffuses away from node $i$ at scale $s$; we select $g_s(\lambda) = \exp(-s\lambda)$ as the heat kernel. This means that we can vary the scaling parameter $s$ to adjust the neighborhoods surrounding a center node. Please see \ref{fig:wavelet} for more illustration. Following \citet{10.1063/5.0152833}, a set of $k$ graph Wavelet $\{\psi\}_i^k \in \mathbb{R}^{n\times n}$ can be constructed to result in a tensor of graph Wavelet $\bold{W} \in \mathbb{R}^{n \times n \times k}$. Unlike structural encoding like random walk which only contains information at one hop length, Wavelets structural information decays gradually from each node to its neighborhood, and, thus, contains much more meaningful topological information.

\paragraph{Fast Graph Wavelet Transform} Traditionally, computing graph wavelet bases necessitates the full diagonalization of the graph Laplacian
$\widetilde{L}$ in \ref{eq:eigen_decomposition}. This approach becomes computationally expensive when the number of nodes increases. Our work adopts the efficient algorithm proposed by \citet{hammond2011wavelets}. This method utilizes Chebyshev polynomials to approximate the Wavelet with a time complexity of $O(|E| \times M)$, where $M$ represents the order of the polynomials and $E$ is the set of edges. This approach scales linearly with the number of edges, thereby being more efficient than previous methods based on the eigendecomposition of the graph Laplacian. 

\subsection{Constructing Long-range Pretraining on Domain-Agnostic Data} \label{sec:pretraining}
\begin{figure}[t]
    \centering
    \includegraphics[width = \textwidth]{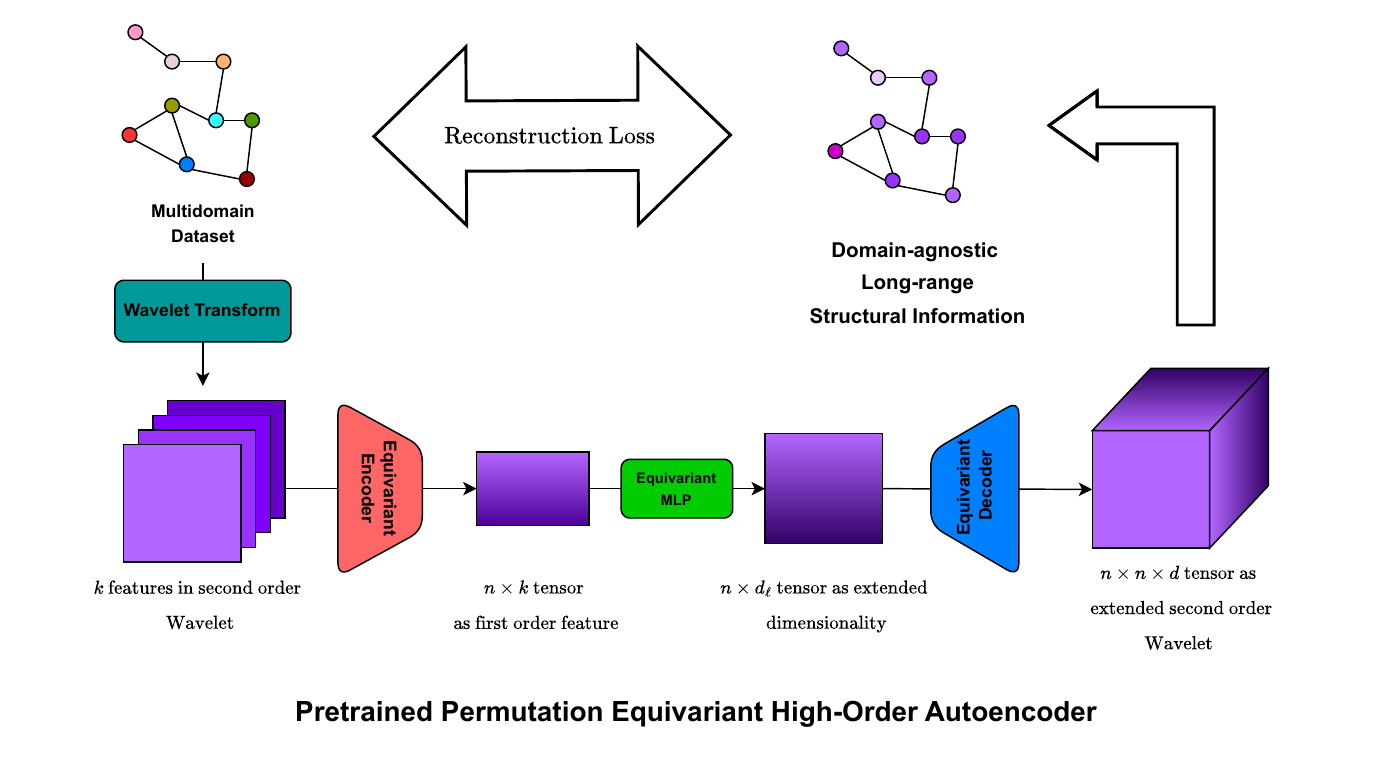}
    \caption{Our proposed equivariant autoencoder pretraining scheme are applied on a large multiple domain dataset while extending the feature degree, overcoming the domain-specific weakness while also embedding long-range information.}\vspace{-5pt}
    \label{fig:AE}
\end{figure}

Our goal is to train a learnable structural encoder to extract generalized and abstract node-level structural features that can be transferred across downstream tasks in graph learning. To fulfill this, we parameterize a high-order autoencoder \cite{Hy_2023, thiede2020general, HyEtAl2018, hy2019covariant, kondor2018covariant} that learns to reconstruct high-order features of graph data from their wavelet tensors $\bold{W}\in\mathbb{R}^{n\times n\times k}$.

In particular, given a second-order wavelet tensor $\bold{W} \in \mathbb{R} ^ {n \times n \times k}$, the encoder $\mathcal{E}$ encodes $\bold{W}$ into a latent matrix $\bold{Z} = \mathcal{E}(\bold{W}) \in \mathbb{R}^{n\times d_\ell}$, the encoder $\mathcal{E}$ can be composed of many equivariant operators. Furthermore, to reduce redundancy caused by high-order training, we extract only two equivariant mappings for the encoder, diagonal and row sum:
\begin{align*}
    \textsc{e}_1(\bold{Z})=\operatorname{diag}(\bold{Z}),\qquad
    \textsc{e}_2(\bold{Z})=\bold{Z}\bold{1}_n.
\end{align*}

Then, the decoder $\mathcal{D}$ lifts $\bold{Z}$ back to a high-order feature map $\mathcal{F}= \mathcal{D}(\bold{Z}) \in \mathbb{R} ^ {n \times n \times d}$. Here, We use the outer product and diagonal operator, which represent structural and positional informationn respectively:
\begin{align*}
    \textsc{d}_1(\bold{Z}_i) = \bold{Z}_i\otimes\bold{Z}_i^{\top},\qquad \textsc{d}_2(\bold{Z}_i) = \operatorname{diag}(\bold{Z}_i)
\end{align*}
where $\otimes$ is the outer product operator, and $\bold{Z}_i$ is the $i$-th channel feature of the encoded information. Here, $d_\ell$ and $d$ denote the channels (dimensions) of each entries in an $n\times n$ matrix.

Finally, we use a channel-wise multi-layer perceptron (MLP) $\phi: \mathbb{R}^{n\times n\times d} \mapsto \mathbb{R}^{n \times n \times r}$ to map $\mathcal{F}$ to a concatenated high-degree adjacency matrix in binary values. Specifically, let $\bold{A}_j$ be the binary matrix of $j$-hop neighbor in a graph and $\widehat{\bold{A}}_j$ be its prediction. The final MLP network returns the predicted array 
\begin{equation}
    \label{predictarrayAj}
    \phi(\bold{Z}) = \left[\begin{array}{cccc}
         \widehat{\bold{A}}_{s_1}& \widehat{\bold{A}}_{s_2} & \dots & \widehat{\bold{A}}_{s_r}
    \end{array}\right],
\end{equation}
where $s_1, s_2, \dots, s_r$ are natural degrees to be chosen. In this work, we let these values follow a exponential pattern, which highlights the range-diversity.

Theoretically, we show that with sufficient budget, our pretraining schema can reach abitraily high accuracy, the full proof is provided at Appendix~\ref{app:approxAd}.

\begin{theorem}
\label{theory:onlytheorem}
    For any $\epsilon>0$ and real coefficients $\theta_1,\theta_2,\dots,\theta_d$, there exists a HOPE-WavePE $\varphi:\mathbb{R}^{n\times n\times d}\to\mathbb{R}^{n\times n}$ such that $$\lrVert{\varphi(\bold{Z})-\sum_{j=1}^r\theta_j\bold{A}_j}<\epsilon.$$
\end{theorem}
\paragraph{Masking grants generalizability} In sparse structures, learning different hop adjacencies imposes disproportionality in edges and non-edges, which differs with the adjcency degree and also data structures. For example, a molecule graph from MoleculeNet \cite{wu2018moleculenet} have shorter diameter than peptides \cite{long_range}. Thus, long degree adjacency is redundant for small molecules and hurts the generalizability for bigger graphs. To fix this, we use a binary mask $\bold{M}\in\mathbb{R}^{n\times n\times r}$, filtering out random entries in both the ground truth and label such that non edges and edges are equal in quantity for each degree of predicted adjacency. In particular, denote $\mathcal{M}(\bold{P})$ and $\mathcal{N}(\bold{P})$ as the number of non-zero and zero entries in matrix $\bold{P}$ respectively, our formula for the masked value quantity is 
\begin{equation}
\label{maskquantity}
    \mathcal{M}(\bold{M}_i)=\min\left(\mathcal{M}(\bold{A}),\mathcal{N}(\bold{A}), \bold{T}\right),
\end{equation}
where $\bold{T}$ is some fixed threshold. The remaining entries define the binary cross entropy loss:
$$\mathcal{L}_{\bold{M}}(\bold{A}_{s_i},\widehat{\bold{A}}_{s_i}) = \operatorname{BinCrossEntropy}\left(\bold{M}_i\odot\bold{A}_{s_i}, \bold{M}_i\odot\widehat{\bold{A}}_{s_i}\right),$$
where $\odot$ is the elementwise matrix product. There are two main motivations for this. \textbf{1)} Firstly, thresholding adjacency channels reduces training cost while keeping a fair learning quantity between hop lengths. \textbf{2)} Secondly, the masked AE can learn long hops for big graphs independently from smaller graphs which are masked out based on \eqref{maskquantity}. This renders the autoencoder \textit{range-awareness} of the input graph in downstream tasks and augment meaningful information to node features.
\vspace{-5pt}
\section{Experiments} \label{sec:experiments}
This section empirically evaluates the advantages of our structurally pre-trained autoencoder (AE). We first detail the pre-training process, which utilizes a large dataset of graph structures. We then present experimental results on various downstream tasks. As our objective is to learn a structural feature extractor for graph data, we aim to demonstrate three key points:
\begin{itemize}
    \item[\textbullet] \textbf{Enhanced performance on small datasets}: The encoder of the pre-trained AE can extract structural features that enhance the performance of graph-level prediction tasks on small datasets.
    \item[\textbullet] \textbf{Effective PSE feature representation}: The learned structural features can capture global information and preserve graph structures when integrated into and fine-tuned along with global models like graph transformers or VN-augmented MPNNs.
    \item[\textbullet] \textbf{Generalizability}: The pre-trained AE can generalize well to out-of-distribution graphs with different connectivity types and sizes from the training dataset.
\end{itemize}

\subsection{Setup}
\paragraph{Pretraining} We pretrained a high-order autoencoder on MolPCBA \cite{wu2018moleculenet} and Peptides-func \cite{long_range}. During this pretraining stage, the autoencoder focuses on learning a set of topological hops, represented by the concatenated tensor $\{\bold{A}_{s_i}\}_{i=1}^r$. By excluding chemical features during pretraining, we granted the network versatility, enabling it to adapt to downstream tasks that use different feature representations. To incorporate multi-scale information, we constructed a 4-channel Wavelet tensor for each graph sample, with scaling factors of $[1,2,4,16]$. The autoencoder architecture consisted of an encoder and a decoder, each containing three high-order linear layers. Another three-layer multilayer perceptron (MLP) was used to project the encoder's output into a latent embedding of size 20. After being projected back by the decoder to generate new scaled wavelets, a MLP is used to learn the adjacency tensor $\{\bold{A}_{s_i}\}_{i=1}^r$. We divide the MolPCBA dataset into a train-valid ratio of 9:1 and use the prepaired train-valid set for Peptides-func.
\vspace{-5pt}
\paragraph{Downstream task} After pretraining, we integrated the pretrained autoencoder into other graph neural networks for graph-level supervised tasks. We concatenated the node-level structural features extracted by the pretrained autoencoder with domain-specific features before feeding them into GNN-baed models for downstream evaluation. The implementation details for each specific downstream task are provided in Appendix~\ref{sec:implementation}.

\subsection{Results}
While we pretrained our HOPE-WavePE autoencoder on the MolPCA and Peptides-func, the performance on downstream tasks within this benchmark may be limited to the chemical structures. To address this, we also evaluate HOPE-WavePE on a broader range of real-world graph datasets. These datasets encompass diverse graph structures, connectivity types, sizes, and domain-specific features, significantly differing from those found in molecules and peptides, highlighting how range-awareness adapt to extreme out-of-distribution (OOD) scenerios.

\paragraph{Moleculenet and Long Range Graph Benchmark} We first evaluate our encoding schema on the two data structures used for pretraining phase. In particular, we use scaffold splitting on five small molecule datasets from Moleculenet \cite{wu2018moleculenet} and two datasets Peptides-func and Peptides-struct from Long Range Graph Benchmark (LRGB) \cite{long_range}, which are shown in Table~\ref{tab:moleculenet_result} and \ref{tab:LRGB}. For Moleculenet, we compare our methods with multiple baselines, including supervised and pretraining. For LRGB, we compare our method against some standard message-passing neural networks and show that by augmenting using virtual node and HOPE-WavePE, they can outperform many Graph Transformer methods.

% We follow the setting in \cite{sun2022does} wherein pretrained models are evaluated on the data pretrained. This approach assesses a model's ability to predict properties of out-of-distribution molecular structures. We compare our methods with multiple baselines, including supervised and pretraining. In particular, two supervised methods are D-MPNN \cite{yang2019analyzing} and AttentiveFP \cite{xiong2019pushing}. For pretraining, we compare our work with PretrainGNN \cite{strategies_pretrain}, GraphCL \cite{GraphCL}, GraphLog \cite{GraphLog}, Grover \cite{GROVER}, and InfoGraph \cite{InfoGraph}. According to the experimental results in \ref{tab:moleculenet_result}, pretrained HOPE-WavePE can improve the performances of MPNN on small molecule property prediction tasks.

\begin{table}[]
\centering
\small
\caption{Experimental results on five small molecule (scaffold-split) classification datasets. Methods are evaluated by ROC-AUC \% ($\uparrow$) where higher scores mean better performances. We report the mean and standard deviation (in brackets) of all methods over three random seeds. Top 3 results are highlighted, including \first{First}, \second{Second}, and \third{Third}.}
\begin{tabular}{@{}llccccc@{}}
\toprule
\textbf{Method} & \textbf{BBBP} & \textbf{BACE} & \textbf{Tox21} & \textbf{ToxCast} & \textbf{SIDER} \\ 
\midrule

D-MPNN \cite{yang2019analyzing} & \third{71.0(0.3)} & 80.9(0.6) & 75.9(0.7) & \third{65.5(0.3)} & 57.0(0.7) \\ 

AttentiveFP \cite{xiong2019pushing} & 64.3(1.8) & 78.4(0.0) & \third{76.1(0.5)} & 63.7(0.2) & 60.6(3.2)  \\

\midrule

N-GramRF \cite{N_gram}  & 69.7(0.6) & 77.9(1.5) & 74.3(0.4) & - & \first{66.8(0.7)} \\

N-GramXGB \cite{N_gram}  & 69.1(0.8) & 79.1(1.3) & 75.8(0.9) & - & \second{65.5(0.7)}  \\

PretrainGNN \cite{strategies_pretrain} & 68.7(1.3) & {79.9(0.9)} & \second{76.7(0.4)} & \second{65.7(0.6)} & 62.7(0.8) \\
                             
GROVERbase \cite{GROVER} & 70.0(0.1) & \third{82.6(0.7)} & 74.3(0.1) & 65.4(0.4)  & {64.8(0.6)} \\
                             
GROVERlarge \cite{GROVER} & 69.5(0.1) & 81.0(1.4) & 73.5(0.1) & 65.3(0.5) & \third{65.4(0.1)} \\

GraphLOG pretrained \cite{GraphLog} & \first{72.5(0.8)} & \second{83.5(1.2)} & 75.7(0.5) & 63.5(0.7) & 61.2 (1.1)\\

GraphCL pretrained \cite{GraphCL} & 69.7(0.7) & 75.4(1.4) & 73.9(0.7) & 62.4(0.6) & 60.5(0.9) \\

InfoGraph pretrained \cite{InfoGraph} & 66.3(0.6) & 64.8(0.7) & 68.1(0.6) & 58.4(0.6) & 57.1(0.8) \\

\midrule                 
MPNN + HOPE-WavePE (ours) & \second{71.2(0.4)} & {\first{86.8(0.4)}} &                                       {\first{78.0(0.1)}} & \first{66.9(0.5)} & {64.7(0.3)} \\ 
\bottomrule
\end{tabular}
\label{tab:moleculenet_result}
\end{table}

\begin{table}[h]
\centering 

\small
\caption{Experimental results on Peptides func and Peptides struct. We report the means and standard deviations of 4 random seeds. The best results are highlighted \first{First}, \second{Second} and \third{Third}}
\begin{tabular}{@{}lcc@{}}
\toprule
\multicolumn{1}{l}{\textbf{Model}} &
\begin{tabular}[c]{@{}c@{}}\textbf{Peptides Struct}\\ \textbf{MAE} $\downarrow$\end{tabular} & \begin{tabular}[c]{@{}c@{}}\textbf{Peptides Func}\\ \textbf{AP} $\uparrow$ \end{tabular} \\ \midrule
GCN \cite{gcn}& 0.3496 ± 0.0013 & 0.5930 ± 0.0023 \\ 
GINE \cite{gin} & 0.6346 ± 0.0071 & 0.5498 ± 0.0079 \\
GatedGCN \cite{gated_gcn} & 0.3420 ± 0.0013 & 0.5864 ± 0.0077 \\
GatedGCN + RWSE \cite{gated_gcn} & 0.3357 ± 0.0006 & 0.6069 ± 0.0035  \\
\midrule 
Drew-GCN + LapPE \cite{drew} & 0.2536 ± 0.0015  & \second{0.7150 ± 0.0044}  \\
GraphGPS + LapPE \cite{gps} & 0.2500 ± 0.0005 & 0.6535 ± 0.0041 \\ 
GRIT \cite{RWPE}         & \second{0.2460 ± 0.0012}  & \third{0.6988 ± 0.0082}   \\

Graph VIT \cite{graphmlpmixer}      & \third{0.2449 ± 0.0016}  & 0.6942 ± 0.0075   \\
GraphMLP Mixer \cite{graphmlpmixer}   & 0.2475 ± 0.0015  & 0.6921 ± 0.0054    \\

GatedGCN + VN + RWSE \cite{cai2023connection}  & 0.2529 ± 0.0009  & 0.6685 ± 0.0062    \\ 

\midrule

GatedGCN + VN + HOPE-WavePE (ours)  & \first{0.2433 ± 0.0011} & \first{0.7171 ± 0.0023} \\ 

\bottomrule
\end{tabular}
\label{tab:LRGB}
\end{table}

% \vspace{-5pt}\paragraph{Long-range Graph Benchmark (LRGB)} To show that 
% pretrained AE can capture long-range information, we conducted experiments on two datasets: Peptide-struct and Peptides-func in the Long-range Graph Benchmark \cite{long_range}. In this task, we integrate pretrained HOPE-WavePE into a VN-augmented MPNN and fine-tune on the two datasets. We compare our work with several MPNNs and transformer-based baselines, including GCN \cite{gcn}, GatedGCN \cite{gated_gcn}, Drew-GCN + LapPE \cite{drew}, GraphGPS + LapPE \cite{gps}, GRIT \cite{RWPE}, GraphVIT, GraphMLP Mixer \cite{graphmlpmixer} and GatedGCN with virtual node and random walk structural encoding \cite{cai2023connection}.

\paragraph{TUDataset Benchmark} We evaluate HOPE-WavePE on six diverse datasets from the TUDataset benchmark \cite{morris2020tudataset}: a small molecule dataset (MUTAG), two chemical compound datasets (NCI1 and NCI109), a macromolecule dataset (PROTEINS) and two social network datasets (IMDB-B and IMDB-M). In this task, we combine the structural features extracted by the pretrained autoencoder with node domain features and feed them into an MPNN consisting of five GIN \cite{gin} layers. The results in \ref{tab:tu_dataset_result} demonstrate that GIN augmented with HOPE-WavePE significantly outperforms other complicated high-order networks like IGN \cite{maron2018invariant}, CIN, and PPGNs on three out of six datasets.

\begin{table}[h]
    \centering
    % \hspace{-2cm}
    \caption{Experimental results on TU datasets. The methods are evaluated by Accuracy \% ($\uparrow$). The reported results are means and standard deviations of runnings over five random seeds. Top 3 results are highlighted, including \first{First}, \second{Second}, and \third{Third}.}\resizebox{\linewidth}{!}{%
    \begin{tabular}{l|cccc|cc}
    % \begin{adjustwidth}{-.5in}{-.in}
        \toprule
        \textbf{Method} & 
        \textbf{MUTAG} &
        \textbf{PROTEINS} &
        \textbf{NCI1} &
        \textbf{NCI109} &
        \textbf{IMDB-B} & 
        \textbf{IMDB-M} \\
        \midrule       
        %\parbox[t]{1mm}{\multirow{10}{*}{\rotatebox[origin=c]{90}{same splits}}} 
        RWK \cite{RWK} & 
         79.2 $\pm$ 2.1 & 
         59.6 $\pm$ 0.1 & 
         $>$3 days & 
         - & 
         - &
         - \\
        %  -\\
        
        GK ($k=3$) \cite{GK3} &
        81.4 $\pm$ 1.7 & 
        71.4 $\pm$ 0.3 & 
        62.5 $\pm$ 0.3 & 
        62.4 $\pm$ 0.3 &
        - &
        - \\

        PK \cite{PK} & 
         76.0 $\pm$ 2.7& 
         73.7 $\pm$ 0.7 & 
         82.5 $\pm$ 0.5 & 
         - & 
         - & 
         - \\

        WL kernel \cite{WLkernel} &
          90.4 $\pm$ 5.7 & 
          75.0 $\pm$ 3.1 & 
          \first{86.0} $\pm$ 1.8 & 
          - &
          73.8 $\pm$ 3.9 &
          50.9 $\pm$ 3.8 \\

        \midrule
         
        DCNN \cite{DCNN}& 
        -&  
        61.3 $\pm$ 1.6 &
        56.6 $\pm$ 1.0 &
        - &
        49.1 $\pm$ 1.4 &
        33.5 $\pm$ 1.4 \\

        DGCNN \cite{DCGNN}  & 
        85.8 $\pm$ 1.8 & 
        75.5 $\pm$ 0.9 & 
        74.4 $\pm$ 0.5 & 
        - &
        70.0 $\pm$ 0.9 & 
        47.8 $\pm$ 0.9 \\
        
        IGN \cite{maron2018invariant}  &
        83.9 $\pm$ 13.0 &
        {76.6} $\pm$ 5.5 &
        74.3 $\pm$ 2.7 & 
        72.8 $\pm$ 1.5 & 
        72.0 $\pm$ 5.5 & 
        48.7 $\pm$ 3.4 \\
        % -\\
        
        GIN \cite{gin}  & 
        89.4 $\pm$ 5.6 & 
        76.2 $\pm$ 2.8 &
        82.7 $\pm$ 1.7 &
        - & 
        75.1 $\pm$ 5.1 &
        52.3 $\pm$ 2.8 \\

        PPGNs \cite{ppgn} &
        {90.6} $\pm$ 8.7 &
        \second{77.2} $\pm$ 4.7 & 
        83.2 $\pm$ 1.1 & 
        \third{82.2} $\pm$ 1.4 &
        73.0 $\pm$ 5.8 & 
        50.5 $\pm$ 3.6 \\

        Natural GN \cite{naturalGN} &
        89.4 $\pm$ 1.6 &
        71.7 $\pm$ 1.0 &
        82.4 $\pm$ 1.3 &
        - &
        73.5 $\pm$ 2.0 &
        51.3 $\pm$ 1.5 \\

        GSN \cite{GSN} &
        \third{92.2} $\pm$ 7.5 &
        {76.6} $\pm$ 5.0 & 
        {83.5} $\pm$ 2.0 &
        - & % <-- NCI109
        \first{77.8} $\pm$ 3.3 & 
        \first{54.3} $\pm$ 3.3 \\
        
        SIN \cite{SIN}  & 
        -  &
        76.4 $\pm$ 3.3 & 
        82.7 $\pm$ 2.1 &
        - &
        \third{75.6} $\pm$ 3.2 & 
        \third{52.4} $\pm$ 2.9 \\

        {CIN} \cite{CIN}  & 
        \second{92.7} $\pm$ 6.1 &
        \third{77.0} $\pm$ 4.3 &
        \third{83.6} $\pm$ 1.4 &
        \first{84.0} $\pm$ 1.6 &
        \third{75.6} $\pm$ 3.7 & 
        \second{52.7} $\pm$ 3.1 \\

        \midrule 
        GIN + HOPE-WavePE (ours) &
        \first{93.6} $\pm$ 5.8 &
        \first{79.5} $\pm$ 4.81 &
        \second{84.5} $\pm$ 2.0 &
          \second{83.1} $\pm$ 1.9      & 
        \second{76.0} $\pm$ 3.7 &
        \second{52.7} $\pm$ 2.9 \\
        
        \bottomrule

    \end{tabular}
    }
\label{tab:tu_dataset_result}
\end{table}
% \vspace{-15pt}
\vspace{-5pt}\paragraph{Different Graph Connectivity Patterns}We evaluate on two image classification datasets, MNIST and CIFAR-10 in \cite{JMLR:v24:22-0567}. In this task, we augment node superpixel features with their structural features before feeding them into GPS \cite{gps}, a graph transformer model. As shown in \ref{tab:image_result}, GPS + HOPE-WavePE achieves comparable accuracy to other MPNNs and GPS models that rely on different explicit positional encoding methods. 

\vspace{-5pt}\paragraph{ZINC Dataset}We evaluate the performance of HOPE-WavePEs on a subset of the ZINC dataset \cite{zinc} containing 12,000 graphs with up to 38 heavy atoms. We assess its ability to perform molecular tasks using GPS \cite{gps} equipped with HOPE-WavePE. The results in Table~\ref{tab:image_result} our method outperforms the two most popular encoding baselines, LapPE and RWSE, and show statistically significant differences compared to traditional GNN methods.

\begin{table}[h]
    \centering
    \small
    \caption{Experimental results on image classification tasks and ZINC. The results are reported by taking the mean $\pm$ the std over four random seeds. Top 3 results are highlighted, including \first{First}, \second{Second} and \third{Third}.}
    % \hspace{-2cm}
    \label{tab:image_datasets}
    \resizebox{0.7\linewidth}{!}{%
    \begin{tabular}{l|ccc}
        \toprule \multirow{2}{4em}{\textbf{Model}} & \textbf{CIFAR10} & \textbf{MNIST} & \textbf{ZINC} \\
         & \textbf{ACC} $(\%) \uparrow$ & \textbf{ACC} $(\%) \uparrow$ & \textbf{MAE} $\downarrow$ \\
        \midrule GIN \cite{gin}  & 55.26 $\pm$ 1.53 & 96.49 $\pm$ 0.25 & 0.526 $\pm$ 0.051 \\
        % GCN \cite{gcn} & 51.8 $\pm$ 0.7 & - & -\\
        GAT \cite{gat}& 64.22 $\pm$ 0.45 & 95.53 $\pm$ 0.20 & 0.384 $\pm$ 0.007 \\
        GatedGCN \cite{gated_gcn} & 67.31 $\pm$ 0.31 & 97.34 $\pm$ 0.14 & 0.282 $\pm$ 0.015 \\
        Graph MLP-Mixer \cite{graphmlpmixer} & \second{72.46 $\pm$ 0.36}  & \first{98.35 $\pm$ 0.10} & \third{0.077 $\pm$ 0.003}\\
        \midrule GPS + none \cite{GPSE}  & {71.76 $\pm$ 0.01} &{98.05 $\pm$ 0.00} & - \\
        GPS + (RWSE/LapPE) \cite{gps}  & \third{72.30 $\pm$ 0.36} & \third{98.05 $\pm$ 0.13} & \second{0.070 $\pm$ 0.004} \\
        
        \midrule GPS + HOPE-WavePE (ours)& \first{73.01 $\pm$ 0.81} & \second{98.15 $\pm$ 0.11} & \first{0.067 $\pm$ 0.002}\\
        \bottomrule
    \end{tabular}
    }
    \label{tab:image_result}
\end{table}

\subsection{Ablation Study}

We have three focuses for our ablations. First, we explore how does the number of wavelet channels affect the reconstructability of the adjacency tensor. Secondly, we show that cross-domain training only improve performance incrementally. And finally, we show how masked and unmasked training affect the reconstructability of HOPE-WavePE on medium-sized graphs.

% \begin{itemize}
%     % \item Is cross-domain training needed to enhance domain-agnostic downstream tasks?
%     \item Does masking highlight important ranges to learn in the range spectrum?
% \end{itemize}

\paragraph{Wavelet Channel Number} As shown in Figure~\ref{fig:wavelet}, different scaling factors impact the receptive field of each node. Since each resolution is responsible for learning a specific range of hops, more multiresolutions should capture diverse-range dependency. We visualize in Figure~\ref{fig:structacc} by recording the reconstruction accuracies of the concatenated array $\{\bold{A}_{2^i}\}_{i=0}^7$
given $7, 10, 20, 30$ and $60$ wavelet channels respectively. The experiment is conducted on the peptides-struct dataset from the Long Range Graph Benchmark \cite{long_range}, which highlights the long-range learnability.

\begin{figure}
\vspace{-15pt}
    \centering
    % \vspace{-10pt}
    \includegraphics[width=0.7\linewidth]{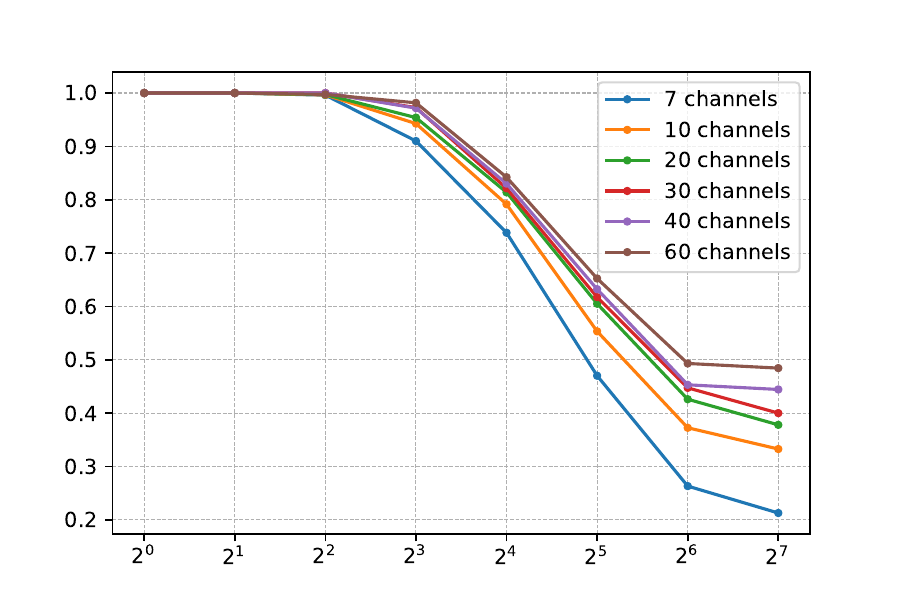}
    \caption{Average reconstruction accuracy on Peptides-struct graphs with different wavelet channel quantities.}
    \label{fig:structacc}
\end{figure}

\begin{figure}[h]
\vspace{-15pt}
    \centering
    % \vspace{-10pt}
    \includegraphics[width=0.7\linewidth]{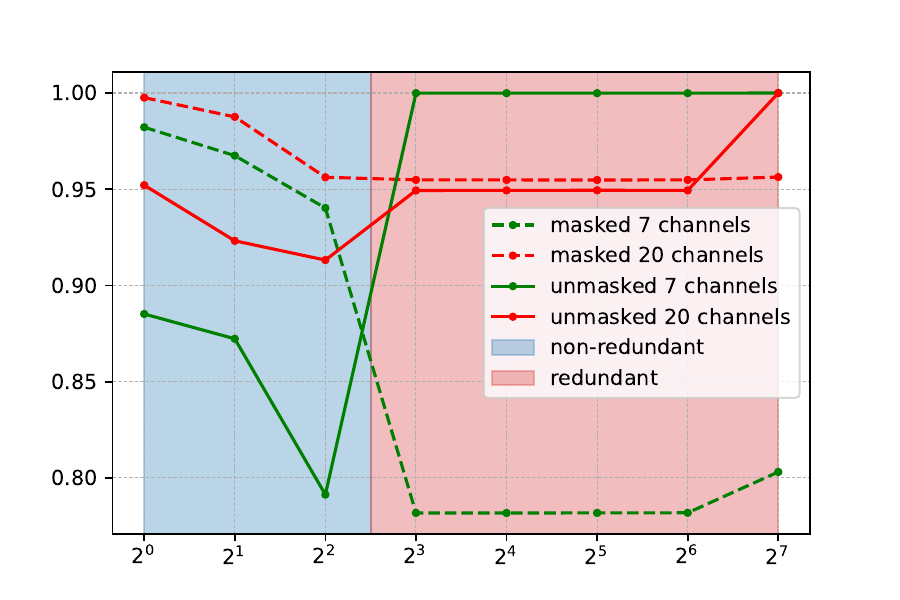}
    \caption{Masked vs unmasked reconstruction accuracy on MNIST dataset. The regions shaded \textcolor{red}{\textbf{red}} and \textcolor{blue}{\textbf{blue}} contains hop lengths $s$ in which $\bold{A}_s$ are and are not all ones, respectively.}
    \vspace{-10pt}
    \label{fig:mask_vs_unmasked}
    % \vspace{=10pt}
\end{figure}

\paragraph{Cross-Domain Training} To further verify transferability of HOPE-WavePE onto different tasks in a direct manner, we train HOPE-WavePE on four different datasets with great domain disparity, namely: ZINC (molecule), MNIST (image), ogbl-collab (collaborative network) and ogbl-ppa (interaction network), and record the decoded accuracy on the reconstruction the adjacency matrix. To avoid resource insufficency in processing the two ogbl tasks while not changing their intrinsic structure, we partitioned them into almost similar size graphs as in MNIST and ZINC using METIS algorithm. Our results shown in Table \ref{abl:crossdomain} indicates the transferrability of HOPE-WavePE to extremely different domains, given that their sizes are similar.

\begin{table}[h]
    \centering
    \small
    \caption{Accuracy of HOPE-WavePE on cross-domain training.}
    % \hspace{-2cm}
    \label{abl:crossdomain}
    \resizebox{0.7\linewidth}{!}{%
    \begin{tabular}{l|l|cccc}
    \toprule
    {\textbf{Domain Type} } & { \textbf{Name} } & \textbf{ZINC} & \textbf{MNIST}&\textbf{ogbl-collab}& \textbf{ogbl-ppa}\\
    \midrule Molecular & ZINC & 0.9891 & 0.9734&0.9812&0.9734 \\
    Image & MNIST & 0.9652 & 0.9753 & 0.9652 & 0.9712\\
    Collaborative Network & ogbl-collab & 0.9860 & 0.9860 & 0.9742 & 0.9812 \\
    Interaction Network & ogbl-ppa & 0.9767 & 0.9767 & 0.9832 & 0.9912\\
    \bottomrule
    \end{tabular}
    }
    
\end{table}

\paragraph{Masked vs Unmasked} To visually illustrate our motivation of masking the autoencoder, we evaluated HOPE-WavePE on reconstructing the concatenated adjacency $\{\bold{A}_{2^i}\}_{i=0}^7$ over the MNIST dataset. Note that, these graphs are connected and their diameter is far below the considered range. In particular, $\bold{A}_{s}=\bold{1}_{n\times n}$ for all $s\geq 8$, and, thus, are redundant in the structural learning. From Figure~\ref{fig:mask_vs_unmasked}, we see that masking significantly removes redundant learning from exceedingly long range, which would punishes important ranges otherwise. 

\section{Conclusion} \label{sec:conclusion}
We have introduced HOPE-WavePE, a novel high-order permutation-equivariant pretraining method specifically designed for graph-structured data. Our approach leverages the inherent connectivity of graphs, eliminating reliance on domain-specific features while being range-aware. This enables HOPE-WavePE to generalize effectively across diverse graph types and domains. The superiority of HOPE-WavePE is demonstrably proven through both theoretical and empirical analysis. Finally, we have discussed the potential of HOPE-WavePE as a foundation for a general graph structural encoder. A promising future direction will be to focus on optimizing the scalability of this approach. 

% For natbib users:
%\newpage
\bibliographystyle{unsrtnat}
\bibliography{paper}
% For bibLaTeX users:
% \printbibliography

\appendix
\newpage
\section{Theoretical analysis} \label{sec:appendix_theory}

% \subsection{$\mathbb{S}_n$-equivariant Encoder + outer product can reconstruct second-order wavelet}

\textbf{Notations.} Throughout this section, for $\bold{X}\in\mathbb{R}^{m\times n}$, we denote $\bold{X}[i:j]$ as the indiced matrix of $X$ from row $i$ to row $j$ and $X[i]$ as the $i$-th row of $\bold{X}$.

\subsection{Preliminary results}

% We make use of an important lemma addressing the expressiveness of ReLU networks.
% \begin{lemma}
%     \label{reluapprox}
%     \cite{relunet} All continuous multivariate functions $f:\mathbb{R}^n\to\mathbb{R}$ can be approximated to arbitrary accuracy by a sufficiently wide two-layer ReLU-based MLP network. Moreover, if $f(\cdot)$ is a piece-wise linear function with $k$ pieces, then the MLP approximators' width is $O(k)$.
% \end{lemma}

% This lemma can be extended to multi-output functions simply by taking one entry at a time. Our proofs use \ref{reluapprox} as an essential statement to handle unlinearizable functions.

The approximation methodology we use is the spline approximation on the eigenvalues of the laplacians, which later . First we need to define a scalar $k$-spline function on a bounded domain $[a,b]\subset \mathbb{R}$. Let 
$$a=:\eta_1<\eta_2<\dots<\eta_{N+1}:=b$$

such that $\eta_{j+1}=\eta_j+(b-a)/N$, these points are called uniform knots in the interval $[a,b]$. Assume that $\mathcal{P}_k$ is the class of polynomial up degree at most $k$.

\begin{definition}
\label{app:ksplinedef}
    A scalar function $f$ is called a $k$-spline function in the uniformly-divided interval $[a,b]$ of $N+1$ knots if the follows conditions are satisfied:
    \begin{itemize}
        \item $f(x)=p_i(x)\qquad\forall\; x\in [\eta_i,\eta_{i+1}],\; i=\overline{1,N}$ for some $p_i\in \mathcal{P}_k$,
        \item Derivative of $f$ up to degree $k$ is continuous.
    \end{itemize}
\end{definition}

For convenience, we denote the class of all such functions in Definition~\ref{app:ksplinedef} as $\mathcal{S}_{k,N}^{[a,b]}$.
\begin{lemma}
    \label{splineapprox}
    \textit{(spline approximation power)} \cite{splineerror} Given a scalar function $u$ of smoothness order $k+1$ in the range $[a,b]$ divided in uniform knots of length $h$, there exists $p\in \mathcal{S}_{k,N}^{[a,b]}$ such that
    $$\lrVert{u-p}\leq \left(\frac{h}{\pi}\right)^k\lrVert{u^{(k+1)}}$$
    where $u^{(k+1)}$ is the $(k+1)$-th order derivative of $u$.
\end{lemma}

Our proof strategy incorporate the approximation of high degree polymonial with $k$-spline methodology. Since both wavelet and random walk second feature embeddings' largest eigenvalue is 1 and smallest eigenvalue is -1, we can let $a=-1$ and $b=1$ in our case. Note that in a two-layer MLP, the number of subintervals $[\eta_i,\eta_{i+1}]$ should correspond to the width in the hidden layer, which is equavalent to $1/h$. Therefore, Lemma~\ref{splineapprox} yields the condition in which a tolerance $\epsilon$ is satisfied.

Letting $u(x)=x^{d}$ for some $d>k$, then $u$ is obviously smooth up to arbitrary order, thus it is also smooth of order $k$. We will let $\lrVert{\cdot}$ be the max norm in the interval $[0,1]$, then we have

\begin{align*}
    \lrVert{u^{(k+1)}}^{1/k}&=\left(d(d-1)\dots(d-k)\right)^{1/k}\sup_{x\in[0,1]}x^{\frac{d-k+1}{k}}\leq d^{1+\frac{1}{k}}
\end{align*}

Combine this with the statement of Lemma~\ref{splineapprox} we deduce that the error defined by the max norm will be less than $\epsilon$ if
$$w:=\frac{1}{h}\geq \pi^{-1}\epsilon^{-\frac{1}{k}}d^{1+\frac{1}{k}} = O\left(\epsilon^{-\frac{1}{k}}d^{1+\frac{1}{k}}\right).$$

From here, we deduce an important lemma:

\begin{lemma}
\label{app:D3}
    Given two natural numbers $d$ and $k$ such that $d>k$, then there exists a two-layer MLP $f:\mathbb{R}^k\to\mathbb{R}$ of hidden width $O\left(\epsilon^{-\frac{1}{k}}d^{1+\frac{1}{k}}\right)$ such that 
    $$\left|x^d-f\left(x,x^2,\dots,x^k\right)\right|<\epsilon$$
\end{lemma}

% From here for any positive integer $r$, we accept the shorthand notations
% \begin{align*}
%     \psi_s^{(1)}_r&=\left[\begin{array}{cccc}
%         \psi_s_{s_1}\bold{1}_n &\psi_s_{s_2}\bold{1}_n &\dots&\psi_s_{s_r}\bold{1}_n   
%     \end{array}\right]^{\top}\\
%     \bold{L}^{(1)}_r&=\left[\begin{array}{cccc}
%         \Tilde{L}\bold{1}_n &\Tilde{L}^2\bold{1}_n &\dots&\Tilde{L}^r\bold{1}_n   
%     \end{array}\right]^{\top}
% \end{align*}

% and the superscripts will denote the feature order, which we will also use later for higher order information.

% Proposition~\ref{splineapprox} indicates the expressiveness of the degree-bounded polynomial subclass, which can be extended to our case of polynomial with respect to $\Tilde{L}$. Moreover, by applying deeper network, the contigency for the lower bound of $k$ can be circumvented.

\begin{lemma}
    \label{extendfirst}
    (First order extension)
    For any $\epsilon>0$ and a given natural number $d>k$, there exists a two-layer $\mathbb{S}_n$-equivariant linear $\operatorname{MLP}:\mathbb{R}^{n\times k}\to\mathbb{R}^{n\times r}$ network with width $O\left(n^{\frac{1}{k}}\epsilon^{-\frac{1}{k}}r^{1+\frac{1}{k}}\right)$ such that $$\lrVert{\operatorname{MLP}\left(\mathbf{E}^{(1)}_k\right)-\mathbf{E}^{(1)}_r}\leq\epsilon.$$
\end{lemma}

\begin{proof}

    Assume that $\Lambda^i$ is the diagonal matrix containing all eigenvalues of $\psi_s^i$. Applying Lemma~\ref{app:D3}, we simply see that $\Lambda^q$ for $q=\overline{k+1,r}$ can be estimated using a two-layer MLP of width $O\left(\epsilon^{-\frac{1}{k}}r^{1+\frac{1}{k}}\right)$ with the max norm error less than $\epsilon$. Formally, let $\widehat{\psi_s^q}$ be the estimation of $\psi_s^q$ and $e_i$ be the error at the $i$-th entry along the diagonal, then we have that

    $$\begin{aligned}
        \lrVert{\widehat{\psi_s^q} - \psi_s^q}=\lrVert{\sum_{i=1}^n e_i\bold{u}_i\bold{u}_i^{\top}}\leq \sum_{i=1}^n|e_i|\lrVert{\bold{u}_i\bold{u}_i^{\top}}\leq n\epsilon
    \end{aligned}$$

    Replacing $\epsilon$ with $\epsilon/n$ yields the desired result.
\end{proof}

With enough first-order informations, i.e. sufficiently large $r$ in Lemma~\ref{extendfirst}, we can reconstruct the second-order features up to arbitrarily high degree.

\begin{theorem}
    \label{firsttosecondorder}
    (First to second order) Assume that $\operatorname{rank}(\psi_s-\mathbf{I}_n)\leq r$ and let $h:\mathbb{R}^{n\times r}\to\mathbb{R}^{n\times n\times r}$ be the resolution-wise outer product, then there exists a broadcasted linear feed forward layer $g:\mathbb{R}^{  r}\to\mathbb{R}^{d}$ such that $(h\circ g)\left(\mathbf{E}_d^{(1)}\right)=\mathbf{E}_r^{(2)}$.
\end{theorem}

\begin{proof}
    The first order features are aggregated through an outer product operator and return $r$ square matrices of order $n$. However, these matrices are all rank one matrices and cannot represent the initial second order features. Since the rank of a square matrix is equivalent to its length minus the multiplicity of the eigenvalue zero, we can see that

    $$\operatorname{rank}(\psi_s-\mathbf{I}_n)=\operatorname{rank}(\psi_s^2-\mathbf{I}_n)=\dots=\operatorname{rank}(\psi_s^d-\mathbf{I}_n)\leq r$$

    Therefore, for all $i=\overline{1,d}$, the matrix $\psi_s^i-\mathbf{I}_n$ can be written as a weighted sum of $r$ rank one matrices produced from the outer product. This concludes the proof.

\end{proof}

\subsection{Proof of Theorem~\ref{theory:onlytheorem}}

\begin{theorem}
    \label{app:reconstruct-lowrank}
    For any $\epsilon>0$ and real coefficients $\theta_0,\theta_1,\dots,\theta_d$ assume that $\operatorname{rank}(\psi_s-\mathbf{I}_n)\leq r$, then there exists an $\mathbb{S}_n$-equivariant AE $f:\mathbb{R}^{n\times n\times k}\to \mathbb{R}^{n\times n\times r}$ of width $O(n^{1/k}r^{1+1/k}\epsilon^{-1/k})$ and a broadcasted feed forward network $g:\mathbb{R}^r\to\mathbb{R}$ such that $$\lrVert{(g\circ f)(\mathbf{E}_k)-\bold{p}(\psi_s)}<\epsilon$$ where $\bold{p}(\psi_s)=\sum_{j=0}^d\theta_j\psi_s^j$.
\end{theorem}

\begin{proof}
\textbf{Encoder.} The input tensor is of size $n\times n\times k$, representing second order feature in $k$ different resolutions. The encoder simply operate resolution-wise and take the row-sum through each square matrix. This encoder will output a first order tensor of size $n\times k$. This layer is evidently $\mathbb{S}_n$-equivariant.

\textbf{Latent.} For the latent space, i.e. first-order feature space, we apply Lemma~\ref{extendfirst} to extend from $k$ resolutions to $r$ resolutions using a two-layer MLP of width $O\left(n^{\frac{1}{k}}\epsilon^{-\frac{1}{k}}r^{1+\frac{1}{k}}\right)$. And since this MLP is also built upon the broadcasting along the $n$-axis, it is also $\mathbb{S}_n$-equivariant.

\textbf{Decoder.} Applying the content of Theorem~\ref{firsttosecondorder} we can conclude the proof.

\end{proof}

% \subsection{Proof of Theorem~\ref{theory:approxAd}}

Once the AE can learn to reconstruct, the following ensures that it can capture long-range information.

\begin{theorem}
\label{app:approxAd}
    For any $\epsilon>0$ and real coefficients $\theta_1,\theta_2,\dots,\theta_d$, there exists a two-layer ReLU feed forward network $\varphi:\mathbb{R}^{n\times n\times d}\to\mathbb{R}^{n\times n}$ of hidden dimension $d_h=2$ such that $$\lrVert{\varphi(\mathbf{E}_d)-\sum_{j=1}^d\theta_j\bold{A}_j}<\epsilon.$$
\end{theorem}
\begin{proof}
    For this proof, we need to consider wavelet and random walk separatedly.
    \textbf{Wavelet.} Let $\psi_s=U\Lambda_sU^{\top}$ where $\Lambda_s=\operatorname{diag}(\exp(-s\lambda_1), \exp(-s\lambda_2), \dots, \exp(-s\lambda_n))$. We first need to perform a transform on the vector basis $\mathbf{E}^{(2)}_{d}$. Essentially, the transformations are done independent of the eigenvectors $U$. Formally, we observe that

    \begin{equation}
        \label{polytoexpo}
            \left(\begin{array}{c}
             e^{-s\lambda_i} - 1  \\
             e^{-2s\lambda_i} -1\\
             \vdots\\
             e^{-ds\lambda_i}-1
        \end{array}\right)\approx \bold{A}\left(\begin{array}{c}
             \lambda_i-1  \\
             (\lambda_i-1)^2\\
             \vdots\\
             (\lambda_i-1)^d
        \end{array}\right)
    \end{equation}
    where $\bold{A}\in\mathbb{R}^{d\times d}$ contains the corresponding Chebyshev polynomial expanding coefficients. Note that \eqref{polytoexpo} is essentially the discrete fourier transform, thus the inversed version is simply

    \begin{equation}
        \label{expotopoly}
        \left(\begin{array}{c}
             \lambda_i-1  \\
             (\lambda_i-1)^2\\
             \vdots\\
             (\lambda_i-1)^d
        \end{array}\right)\approx\bold{A}^{-1}\left(\begin{array}{c}
             e^{-s\lambda_i} - 1  \\
             e^{-2s\lambda_i} -1\\
             \vdots\\
             e^{-ds\lambda_i}-1
        \end{array}\right).
    \end{equation}

    This means that the power of $\Tilde{L}$ up to $d$ can be retrieved via a broadcasted linear layer. Now let $\zeta$ be the scalar step function, meaning $\zeta(x)=1$ for $x>0$ and 0 otherwise. Then, let $\varphi$ be a continuous piece-wise linear function such that:
    $$\varphi(x)=\left\{\begin{array}{cl}
        0&\quad\text{if }x<0\\
        x/\varepsilon&\quad\text{if } 0\leq x<1\\
        1&\quad\text{if }x\geq 1
    \end{array}\right.$$
    Since this function is a three-piece linear function, it can be represented as a ReLU-based feed forward network with hidden dimension two. And evidently,
    $$\lrVert{\varphi-\zeta}\to 0\quad \text{as}\quad\varepsilon\to 0.$$

    Furthermore, it yields that $\zeta\left((\bold{I}_n - \Tilde{L})^k\right)=\bold{A}_k$ for all $k$. Therefore, we concluded the proof.
\end{proof}
\section{Additional implementation details} \label{sec:implementation}

\subsection{Datasets}
Table~\ref{tab:dataset_details} presents details of all benchmarking datasets used in our experiments. We focus on improving model performance in graph-level prediction tasks. All datasets contain over 1000 samples, with the average number of nodes per dataset ranging from 13 to over 100.
% The details for all benchmarking datasets are listed in Table~\ref{tab:dataset_details}. 

\begin{table}[h]
    \centering
    \resizebox{\linewidth}{!}{\begin{tabular}{l|cccccr}
    \toprule \textbf{Dataset} & \textbf{\# Graphs} & \textbf{\# Nodes} & \textbf{\# Edges} & \textbf{Pred. level} & \textbf{Pred. task} & \textbf{Metric} \\
    % \toprule ZINC-subset & 12,000 & 23.15 & 24.92 & graph & reg. &  MAE \\
    \midrule CIFAR10 & 60,000 & 117.63 & 469.10 & graph & class. (10-way) &  ACC \\
    MNIST & 70,000 & 70.57 & 281.65 & graph & class. (10-way) &  ACC \\
    \midrule ZINC-subset & 12,000 & 23.15 & 24.92 & graph & reg. &  MAE \\
    %MolHIV & 41,127 & 25.51 & 27.46 & graph & class. (binary) &  AUROC \\
    %MolPCBA & 43,929 & 25.97 & 28.11 & graph & class. (binary) &  AP \\
    MolBBBP & 2,039 & 24.06 & 25.95 & graph & class. (binary) &  AUROC \\
    MolBACE & 1,513 & 34.09 & 36.86 & graph & class. (binary) &  AUROC \\
    MolTox21 & 7,831 & 18.57 & 19.29 & graph & class. (binary) &  AUROC \\
    MolToxCast & 8,576 & 18.78 & 19.26 & graph & class. (binary) &  AUROC \\
    MolSIDER & 2,039 & 33.64 & 35.36 & graph & class. (binary) &  AUROC \\
    % PCQM4Mv2-subset & 446,405 & 14.15 & 14.58 & graph & reg. & 1 & MAE \\
    \midrule Peptides-func & 15,535 & 150.94 & 153.65 & graph & class. (binary) &  AP \\
    Peptides-struct & 15,535 & 150.94 & 153.65 & graph & reg. &  MAE \\
    \midrule 
    MUTAG & 188 & 17.9 & 39.6 & graph & class. (binary) & ACC \\
    PROTEINS & 1,113 & 39.1 & 72.8& graph & class. (binary) & ACC \\
    NCI1 & 4,110 & 29.9 & 32.3 & graph & class. (binary) & ACC \\
    NCI109 & 4,127 & 29.7 & 32.1 & graph & class. (binary) & ACC \\
    IMDB-B & 1,000 & 19.8 & 96.5 & graph & class. (binary) &ACC\\
    IMDB-M & 1,500 & 13.0 & 65.9 & graph & class. (3-way) &ACC\\
    % CSL & 150 & 41.00 & 82.00 & graph & class. (10-way) & 1 & ACC \\
    % EXP & 1,200 & 48.70 & 60.44 & graph & class. (binary) & 1 & ACC \\
    % arXiv & 1 & $169 \mathrm{~K}$ & $40 \mathrm{M}$ & node & class. (40-way) & 1 & ACC \\
    % Proteins & 1 & $133 \mathrm{~K}$ & $1.2 \mathrm{M}$ & node & class. (binary) & 112 & AUROC \\
    \bottomrule
    \end{tabular}}
    \vspace{10pt}
    \caption{Dataset details for transferability experiments on image, ZINC, MoleculeNet, LRGB and TUDataset.}
    \label{tab:dataset_details}
\end{table}

\subsection{Hyperparameter Settings}
\paragraph{Pretraining} \ref{tab:pretrain_setting} depicts the hyperparameters of our high-order autoencoder and training settings. In general, we used three layers of IGN \cite{maron2018invariant} to build the encoder hidden dimensions of $[8,16, 32]$. The decoder is a reversed of encoder with hidden dimensions $[32, 16, 8]$. We used a channel-wise 2-layer MLP to compute the latent $\bold{Z}$ from the encoder's output, and the latent dimension is set to $20$. We preprocessed the Wavelet signals of graph data via the PyGSP \cite{defferrard2017pygsp} software. For each graph, we performed Wavelet transform to get its 4-resolution Wavelet tensor, where each scale varies in [0.25, 0.5, 0.75, 1]. Finally, the autoencoder is pretrained in 100 epochs with a batch size of $32$ and learning rate of $0.0005$.

\begin{table}[h]
\small
\centering 
\begin{tabular}{@{}ccccccccc@{}}
\toprule
\textbf{Batch size} & \textbf{\# Epoch} & \textbf{Encoder}         & \textbf{Decoder}       & \textbf{Learning rate}  & \textbf{Scales}                     & \textbf{Latent dim} & \textbf{ Masking Threshold} \\ \midrule
32         & 100      & {[}8, 16, 32{]} & {[}32,16,8{]} & 0.0005 & {[}1,2,4,8{]} & 20 & 100        \\ \bottomrule
\end{tabular}
\vspace{10pt}
\caption{Hyperparameter settings for pretraining high-order AE.}
\label{tab:pretrain_setting}
\end{table}

\paragraph{MoleculeNet} \ref{moleculenet_setting} shows the hyperparameter settings for fine-tuning MPNN on five downstream datasets in MoleculeNet benchmark. In general, we used local attetion as proposed in \cite{ijcai2021p214}. To model the global interactions, we augment virtual nodes to the local models to improve the performances in ToxCast and SIDER. 

\begin{table}[h]
\centering
\begin{tabular}{@{}lccccc@{}}
\toprule
\textbf{Hyperparameter} & \textbf{BBBP} & \textbf{BACE} & \textbf{Tox21} & \textbf{ToxCast}  & \textbf{SIDER}           \\ \midrule
Pre MPNN & MLP & MLP & MLP & MLP & MLP \\
MPNN type & Attention & Attention & Attention & Attention & Attention \\
VN Augmented & -               & -               & -               & \checkmark               & \checkmark               \\
\# Layers    & 5               & 5               & 5               & 3               & 3               \\
Hidden Dim   & 300             & 300             & 300             & 512             & 512             \\

Dropout      & 0.5             & 0.5             & 0.5             & 0.5             & 0.5             \\
Pooling type & mean            & mean            & mean            & mean            & mean \\
Learning rate & $1e-3$ & $1e-3$ & $1e-3$ & $1e-3$ & $1e-3$ \\
Weight decay & $1e-9$ & $1e-9$ & $1e-9$ & $1e-9$ & $1e-9$ \\
\# Epochs     & 50              & 50              & 50              & 100            & 50              \\
Batch size & 32 & 32 & 32 & 32 & 32 \\
\bottomrule
\end{tabular}
\vspace{10pt}
\caption{Hyperparameter settings for downstream evaluations on the MoleculeNet Benchmark.}
\label{moleculenet_setting}
\end{table}

\paragraph{TUDataset} Table~\ref{tab:tudataset_setting} summarizes the hyperparameter settings for the transfer learning experiments on six TUDataset benchmark datasets. For IMDB-B and IMDB-M, which lack domain node features, we employed HOPE-WavePE as their initial node features. To create unified node features compatible with the hidden dimensions of the MPNN layers, we added a 2-layer MLP before the MPNN layers to update the combination of domain and HOPE-WavePE features. Following \citet{CIN}, we performed 10-fold validations for each dataset and reported means and standard deviations. 

\begin{table}[h]
\centering
\small
\begin{tabular}{@{}lcccccc@{}}
\toprule
\textbf{Hyperparameter} & \textbf{MUTAG} & \textbf{PROTEINS} & \textbf{NCI1} & \textbf{NCI109} & \textbf{IMDB-B} & \textbf{IMDB-M} \\ \midrule
Node Feat & Domain + PE & Domain + PE & Domain + PE & Domain + PE & PE & PE \\
Pre MPNN & MLP & MLP & MLP & MLP & MLP & MLP \\
\# MPNN Layers      & 5     & 5        & 5    & 5      & 5      & 5      \\
Hidden Dim     & 32    & 32       & 32   & 128    & 128    & 128    \\
\# Epochs      & 100   & 100      & 200  & 200    & 100    & 200    \\
Batch size     & 128   & 128      & 128  & 128    & 128    & 128    \\ 
Learning rate & $1\mathrm{e}-3$ & $1\mathrm{e}-3$ & $1\mathrm{e}-3$ & $1\mathrm{e}-3$ & $1\mathrm{e}-3$ & $1\mathrm{e}-3$ \\
Dropout & 0.5 & 0.5 & 0.5 & 0.5 & 0.5 & 0.5 \\
Graph pooling & mean & mean & mean & mean & mean & mean \\
\bottomrule
\end{tabular}
\vspace{10pt}
\caption{Hyperameter settings for downstream evaluations on the TUDataset benchmark.}
\label{tab:tudataset_setting}
\end{table}

\paragraph{ZINC, Image Classification Tasks, and LRGB} We follow the best hyperparameter settings issued in previous work of GPS \cite{gps} and MPNN+VN \cite{cai2023connection}; then, we fine-tuned for better performance. Our full hyperparameter studies of the benchmarks are shown in Table~\ref{tab:implementation_table}.

\begin{table}[h]
\small
    \centering
    \begin{tabular}{lccccc}
        \toprule \textbf{Hyperparameter} & \textbf{ZINC (subset)} & \textbf{MNIST} & \textbf{CIFAR10} & \textbf{Peptides-func} & \textbf{Peptides-struct} \\
        \toprule \# Layers & 9 & 3 & 3 & 3 & 3 \\
        Global Model & Transformer & Transformer & Transformer & Virtual Node & Virtual Node \\
        Local Model & GINE & GatedGCN & GatedGCN & GINE & GINE \\
        Hidden dim & 64 & 50 & 50 & 128 & 128 \\
        \# Heads & 8 & 4 & 4 & - & - \\
        Dropout & 0 & 0.2 & 0.2 & 0 & 0.01 \\
        % Attention dropout & 0.2 & 0.5 & 0.5 & - & - \\
        Graph pooling & sum & mean & mean & mean & mean \\
        \toprule PE dim & 20 & 20 & 20 & 20 & 20 \\
        Node Update & MLP & MLP & MLP & MLP & MLP \\
        \toprule Batch size & 128 & 128 & 128 & 32 & 32 \\
        Learning Rate & 0.0005 & 0.003 & 0.004 & 0.0005 & 0.0005 \\
        \# Epochs & 3000 & 300 & 300 & 100 & 100 \\
        \# Warmup epochs & 50 & 5 & 5 & 5 & 5 \\
        Weight decay &$3\mathrm{e}-5$ & $2\mathrm{e}-4$ & $3\mathrm{e}-5$ & $1\mathrm{e}-5$ & $1\mathrm{e}-5$ \\
        \toprule \# Parameters & 452,299 & 150,081 & 142,093 & 477,953 & 432,206 \\
        \toprule
    \end{tabular}
    \vspace{10pt}
    \caption{Hyperparameter settings for ZINC, MNIST, CIFAR10, Peptides-func and Peptides-struct dataset}
    \label{tab:implementation_table}
\end{table}

% \begin{table}[h]
%     \centering
%     \begin{tabular}{lccccc}
%         \toprule Hyperparameter & ZINC (subset) & MNIST & CIFAR10 & Peptides-func & Peptides-struct \\
%         \toprule \# Transformer Layers & 10 & 3 & 3 & 10 & 16 \\
%         Hidden dim & 64 & 52 & 52 & 64 & 48 \\
%         \# Heads & 8 & 4 & 4 & 8 & 8 \\
%         Dropout & 0 & 0 & 0 & 0 & 0.01 \\
%         Attention dropout & 0.2 & 0.5 & 0.5 & 0.2 & 0.5 \\
%         Graph pooling & sum & mean & mean & - & - \\
%         \toprule PE dim (RW-steps) & 21 & 18 & 18 & 21 & 32 \\
%         PE encoder & linear & linear & linear & linear & linear \\
%         \toprule Batch size & 128 & 16 & 16 & 32 & 16 \\
%         Learning Rate & 0.001 & 0.001 & 0.001 & 0.0005 & 0.0005 \\
%         \# Epochs & 2000 & 200 & 200 & 100 & 100 \\
%         \# Warmup epochs & 50 & 5 & 5 & 5 & 5 \\
%         Weight decay & $1 \mathrm{e}-5$ & $1 \mathrm{e}-5$ & $1 \mathrm{e}-5$ & $1 \mathrm{e}-5$ & $1 \mathrm{e}-5$ \\
%         \toprule \# Parameters & 473,473 & 102,138 & 99486 & 477,953 & 432,206 \\
%         \toprule
%     \end{tabular}
%     \vspace{10pt}
%     \caption{Hyperparameter settings for ZINC, MNIST, CIFAR10, Peptides-func and Peptides-struct dataset}
%     \label{tab:implementation_table}
% \end{table}

\end{document}